\documentclass{article} %
\usepackage{times}

\usepackage{amsmath,amsfonts,bm}

\def\eqref#1{equation~\ref{#1}}

\def\1{\bm{1}}

\def\vtheta{{\bm{\theta}}}
\def\va{{\bm{a}}}
\def\vb{{\bm{b}}}

\def\vp{{\bm{p}}}

\def\vx{{\bm{x}}}
\def\vy{{\bm{y}}}
\def\vz{{\bm{z}}}
\def\valpha{{\bm{\alpha}}}
\def\vomega{{\bm{\omega}}}
\def\vtheta{{\bm{\theta}}}

\DeclareMathAlphabet{\mathsfit}{\encodingdefault}{\sfdefault}{m}{sl}
\SetMathAlphabet{\mathsfit}{bold}{\encodingdefault}{\sfdefault}{bx}{n}

\DeclareMathOperator*{\argmin}{arg\,min}

\usepackage{natbib}
\usepackage{hyperref}

\usepackage{url}            %
\usepackage{booktabs}       %
\usepackage{amsfonts}       %
\usepackage{nicefrac}       %
\usepackage{microtype}      %
\usepackage{xcolor}         %

\usepackage{graphicx}
\usepackage{caption}
\usepackage{subcaption}

\usepackage{algorithm}
\usepackage{algpseudocode}

\usepackage{latexsym}
\usepackage{tabularx}
\usepackage{multirow}
\usepackage{enumitem}
\usepackage{pifont}
\usepackage{color, colortbl}
\usepackage{float}
\usepackage{appendix}
\usepackage{bbold}
\usepackage{xspace}
\usepackage{adjustbox}
\usepackage{changepage}
\usepackage[left=3cm,right=3cm]{geometry}
\definecolor{DarkPink}{rgb}{0.5,0.0,0.18}
\definecolor{DarkGreen}{rgb}{0.1,0.5,0.1}
\definecolor{DarkRed}{rgb}{0.5,0.1,0.1}
\definecolor{DarkBlue}{rgb}{0.1,0.1,0.7}
\definecolor{DarkYellow}{rgb}{.79,.79,0}
\hypersetup{
    unicode=false,          %
    pdftoolbar=true,        %
    pdfmenubar=true,        %
    pdffitwindow=false,      %
    pdfnewwindow=true,      %
    colorlinks=true,       %
    linkcolor=DarkBlue,          %
    citecolor=DarkRed,        %
    filecolor=DarkRed,      %
    urlcolor=DarkBlue,          %
    pdftitle={},
    pdfauthor={},
}

\usepackage{listings}
\definecolor{codegreen}{rgb}{0,0.6,0}
\definecolor{codegray}{rgb}{0.5,0.5,0.5}

\definecolor{backcolour}{RGB}{245,248,250}
\definecolor{emph}{RGB}{166,88,53}
\definecolor{nightblue}{RGB}{9,49,105}
\definecolor{keywords}{RGB}{207,33,46}
\definecolor{lightpurple}{RGB}{130,81,223}

\lstdefinestyle{mystyle}{
    backgroundcolor=\color{backcolour},   
    commentstyle=\color{codegreen},
    keywordstyle=\color{keywords},
    stringstyle=\color{nightblue},
    basicstyle=\fontsize{6.5}{6.5}\ttfamily,
    breakatwhitespace=true,         
    breaklines=true,                 
    captionpos=b,                    
    keepspaces=true,                 
    numberstyle=\tiny\color{codegray},
    numbersep=2pt,                  
    showspaces=false,                
    showstringspaces=false,
    showtabs=false,                  
    tabsize=2,
    captionpos=t,
    emph={},
    emphstyle={\color{lightpurple}},
    linewidth=0.98\columnwidth,
    frame=tb,    
    xrightmargin=0pt,
    xleftmargin=0.23cm,
    numbers=left,
    aboveskip=0.4cm,
    belowskip=0.4cm,
}

\usepackage{amsmath}
\usepackage{amssymb}
\usepackage{mathtools}
\usepackage{amsthm}
\newtheorem{theorem}{Theorem}

\title{Dynamic Gradient Alignment for Online Data Mixing}

\author{Simin Fan \\
Apple and EPFL \\
\and
David Grangier \\
Apple \\
\and 
Pierre Ablin \\
Apple \\
}

\begin{document}

\maketitle
\begin{abstract}
The composition of training data mixtures is critical for effectively training large language models (LLMs), as it directly impacts their performance on downstream tasks. Our goal is to identify an optimal data mixture to specialize an LLM for a specific task with access to only a few examples. Traditional approaches to this problem include ad-hoc reweighting methods, importance sampling, and gradient alignment techniques.
This paper focuses on gradient alignment and introduces Dynamic Gradient Alignment (DGA), a scalable online gradient alignment algorithm. DGA dynamically estimates the pre-training data mixture on which the models' gradients align as well as possible with those of the model on the specific task.
DGA is the first gradient alignment approach that incurs minimal overhead compared to standard pre-training and outputs a competitive model, eliminating the need for retraining the model. Experimentally, we demonstrate significant improvements over importance sampling in two key scenarios: (i) when the pre-training set is small and importance sampling overfits due to limited data; and (ii) when there is insufficient specialized data, trapping importance sampling on narrow pockets of data.
Our findings underscore the effectiveness of gradient alignment methods in optimizing training data mixtures, particularly in data-constrained environments, and offer a practical solution for enhancing LLM performance on specific tasks with limited data availability.
\end{abstract}

\section{Introduction}
Large Language Models (LLMs) are typically pre-trained on extensive, generic corpora sourced from a variety of data domains \citep{brown2020languagemodelsfewshotlearners,touvron2023llama,zhang2022opt}, with the composition of these corpora often depending on domain availability or heuristics \citep{gao2020pile,together2023redpajama}. 
While the diversity of natural texts allows the model to learn from various knowledge sources, not all data domains are equally beneficial according to the targeted tasks. The uncurated nature of web-crawled contents could lead to sub-optimal outcomes due to the variations in data quality \citep{longpre2023pretrainer}. Plus, some domains may contain misinformation and biases, as one potential source of hallucinations in language generation \citep{lin2022truthfulqameasuringmodelsmimic,huang2023surveyhallucinationlargelanguage}. 

To better generalize to the downstream target tasks, it is critical to identify the most beneficial pretraining subset from large, generic pretraining corpora. While data selection per sample can be costly, \textit{domain reweighting} offers an efficient group-level selection approach. Domain reweighting methods assume that samples from the same domain share similar features and search for optimal sampling weights across \textit{domains} \citep{xie2023doremi, fan2024doge, liu2024regmix, kang2024autoscale,grangier2024specialized}.
The domains that most positively impact the target tasks should be assigned higher weights.

In this work, on top of a large, generic pretraining corpus, we assume we have access to a few examples representative of the downstream task on which we want the model to generalize, a so-called \emph{specialized set}. 
For this setup, \citet{grangier2024specialized} recently proposed a simple and scalable \emph{importance sampling} based method to domain reweighting, where the weight of a domain is given by the frequency of samples in the specialized set closest to the domain, where distance is measured with SentenceBert~\citep{reimers2019sentenceBERT} embeddings. This method determines the domain weights before any training and is model-agnostic.

Likewise, prior gradient-alignment methods determine a \textbf{\textit{static}} domain weights for large-scale LM training, often relying on a small-scale proxy model \citep{xie2023doremi, fan2024doge} or fitting a scaling law \citep{liu2024regmix, kang2024autoscale}. While these methods show improvements over training on the natural distribution of a generic corpus, they do not dynamically update domain weights during training to adapt to the current model state. In practical training scenarios, a large model may quickly overfit on certain domains with high weights. In such cases, an online weighting method can respond by shifting emphasis to other domains.

We propose Dynamic Gradient-Alignment (\textsc{DGA}), an \textbf{\textit{online}} domain reweighting method that estimates step-wise optimal domain weights during model training. 
Inspired by \textsc{DoGE} \citep{fan2024doge}, at each reweighting step, \textsc{DGA} upweights the data domain whose gradient aligns more with the model's gradient on the specific set. From the optimization perspective, training the model on the most-aligned data domain yields the greatest reduction in the targeted loss.
By incorporating an exponential-moving-average (EMA) term in online domain weights updates, \textsc{DGA} effectively mitigates overfitting and prioritizes the domains that currently benefit the target task the most.
Since the domain weights and model parameters are updated concurrently, inaccurate domain weights can potentially drive the model into suboptimal states, which further leads to snow-balled errors. In such cases, the EMA term serves as a correction factor, guiding the model back to a more stable state. 
As an additional contribution, we scale the domain reweighting methods into extremely fine-grained domains (e.g. $262k$ domains) by introducing a novel \textit{distribution reweighting mechanism}. Rather than directly reweighting $262k$ data domains, distribution reweighting reparameterizes the high-dimensional domain weights as a convex combination of weight vectors derived from a set of distributions estimated from embedding-based importance sampling~\citep{grangier2024specialized}. 
With the number of distributions less than the number of training data domains, it allows \textsc{DGA} to scale to thousands of domains and make the most of the fine-grained group-level features.

Our experiments demonstrate the effectiveness of DGA compared to standard pre-training and importance sampling baselines in two challenging cases: (1) the resource of training tokens in each domain is limited instead of infinite (\S~\ref{sec:token_limit}), and (2) the domain granularity is extremely large, which introduces intractable computation overheads on the domain reweighting problem (\S~\ref{sec:expe_fine_domains}).

\section{Data Mixing with Specialized Target}
\subsection{Generic dataset and specific tasks}

We consider a generic training corpus $D_{\mathrm{gen}} = \{D_1, \hdots, D_k \}$, which is partitioned into $k$ distinct data domains.
We can sample from each of the $k$ domains to train a model. 
Consequently, we can sample from a \emph{mixture} of these domains and draw a batch of data following the law $\vx \sim \mathrm{mix}(\valpha)\triangleq \sum_{i=1}^k\alpha_iD_i$, where $\valpha\in\mathbb{R}^k$ is the mixture \emph{weights},  belonging to the \emph{simplex} $\valpha\in \Delta^k \triangleq \{\valpha\in\mathbb{R}^k| \sum_{i=1}^k \alpha_i = 1 \text{ and } \alpha_i \geq 0 \text{ for all }i\}$. 
Here, getting one sample from $\mathrm{mix}(\valpha)$ means first getting a random index $i\in\{k\}$ from the categorical distribution corresponding to the vector of probabilities $\valpha$, and then outputting a random sample from the domain $D_i$.
Sampling from this law is computationally efficient if we can efficiently sample from each domain.
Next, we consider a model, parameterized by $\vtheta\in\mathbb{R}^p$, and a loss function $\ell(\vtheta, \vx)$ defined for $\vx\in D_{\mathrm{gen}}$.
To simplify notation, given a set of samples $S$ (which can be  either a full dataset $D_i$, or a mini-batch), we denote the average over $S$ of the loss $\ell(\vtheta, S)  \triangleq\frac{1}{\#S}\sum_{\vx\in S}\ell(\vtheta,\vx)$.
Since we focus on LLMs, $\ell$ is typically the next-token-prediction loss.
For a given mixture weight $\valpha$, we can update $\vtheta$ by doing optimization steps on the \emph{generic loss}
\begin{align}
    \label{eq:generic_loss}
    L_{\mathrm{gen}}(\vtheta, \valpha) \triangleq \mathbb{E}_{\vx\sim \mathrm{mix}(\valpha)}[\ell(\vtheta, \vx)] = \sum_{i=1}^k \alpha_iL_i(\vtheta)\text{ with }L_i(\vtheta)\triangleq\ell(\vtheta, D_i)
\end{align}
In this paper, our goal is to use this data-mixture to train a model that performs well a \emph{specific} task.
We assume to have access to samples from this task, split into train and test sets. 
We call the train set the \emph{specific dataset} $D_{\mathrm{spe}}$ that we use to train models.
The performance on the specific set is measured with the \emph{specific loss}
\begin{equation}
    \label{eq:specific_loss}
    L_{\mathrm{spe}}(\vtheta) \triangleq \ell(\vtheta, D_{\mathrm{spe}}).
\end{equation}
We assume that the specific set $D_{\mathrm{spe}}$ is small; hence, optimizing $L_{\mathrm{spe}}$ directly leads to overfitting: the loss on the test data would be much higher than $L_{\mathrm{spe}}$.
Instead, to get to a low specific loss, we aim to find the optimal data mixing across $k$ data domains $\valpha$ at each training step to get a good model while training on the reweighted generic distribution $\mathrm{mix}(\valpha)$. 

The target specialization task can be flexible according to the application domains, ranging from reasoning, instruction following, etc., corresponding to various objective loss functions, including the next-token prediction loss and preference-based losses when applied on pair-wise datasets.
In this paper, we focus on next-token prediction on another dataset.
Next, we introduce a general bilevel formulation of the data mixing problem.

\subsection{Bilevel formulation}
Since the lack of data forbids optimizing directly $L_{\mathrm{spe}}$, we look for the mixture $\valpha$ such that optimizing the generic loss $L_{\mathrm{gen}}(\vtheta, \valpha)$ yields the smallest specific loss~\citep{grangier2023adaptive}. 
This is formalized by the following \emph{bilevel optimization} problem~\citep{bracken1973mathematical,dagreou2022framework}:
\begin{equation}\label{equ:bilevel-form}
    \valpha^{\star} \in \argmin_{\valpha \in \Delta^k} L_{\mathrm{spe}}(\vtheta^{\star}(\valpha)), \text{ such that }\vtheta^{\star}(\valpha) \in \argmin_{\vtheta} L_{\mathrm{gen}}(\vtheta, \valpha)
\end{equation}
This bilevel formulation is intuitive: for a given weight $\valpha$, the parameters obtained by minimizing the generic loss $L_{\mathrm{gen}}$ are $\vtheta^*(\valpha)$, and we want those weights to yield a small specific loss.
Notably, if the specific loss is a mixture of generic data with an unknown weight $\tilde{\valpha}$, the bilevel formulation is guaranteed to recover it. In other words:
\begin{theorem}
    Assume that there exists $\tilde{\valpha}$ such that $D_{\mathrm{spe}} = \mathrm{mix}(\tilde{\valpha})$
    . Then, $\tilde{\valpha}$ is a solution to the bilevel problem in \autoref{equ:bilevel-form}.
\end{theorem}
\begin{proof}
    We let $\tilde{\vtheta}$ the minimizer of $L_{\mathrm{spe}}$. Then, for all $\valpha$, we have by definition that $L_{\mathrm{spe}}(\vtheta^*(\valpha))\geq L_{\mathrm{spe}}(\tilde{\vtheta})$. Furthermore, since $D_{\mathrm{spe}} = \mathrm{mix}(\tilde{\valpha})$, we have that $L_{\mathrm{gen}}(\vtheta, \tilde{\valpha})=L_{\mathrm{spe}}(\vtheta)$ for all $\vtheta$, hence minimizing this yields $\vtheta^*(\valpha) = \tilde{\vtheta}$. Putting these results together, we have proven that for all $\valpha$, it holds $L_{\mathrm{spe}}(\vtheta^*(\valpha))\geq L_{\mathrm{spe}}(\vtheta^*(\tilde{\valpha}))$, so that $\tilde{\valpha}$ is a solution to \autoref{equ:bilevel-form}.
\end{proof}

We consider two types of methods to solve \autoref{equ:bilevel-form}. Static methods construct a single mixture weight vector $\valpha$ and then minimize $L_{\mathrm{gen}}(\vtheta, \valpha)$; 
we describe in the next section how to obtain this vector $\valpha$. 
Online methods modify the weights dynamically during model training. They
produce a sequence of weights $\valpha^{(t)}$ where $t$ is the optimization iterate. In that case, at each training step, the parameters $\vtheta^{(t)}$ are updated by doing an optimization step --- with gradient descent or Adam --- on the function $L_{\mathrm{gen}}(\vtheta, \valpha^{(t)})$.
We now discuss methods to obtain a weight vector $\valpha$ or a sequence $\valpha^{(t)}$.

\subsection{A Strong Baseline: Importance Sampling}
\label{sec:importance_sampling}
A sensible strategy is to train the model on a data mixture that most resembles the composition of the targeted specialization data distribution. 
This is the philosophy behind importance sampling~\citep{kloek1978bayesian}.
We estimate the importance sampling weights $\valpha^{\mathrm{IS}}$ using the method of \citet{grangier2024specialized}.
The core idea is to embed each generic domain using SentenceBert~\citep{reimers2019sentenceBERT}, and then compute the centroid of each domain $\vb_i = \frac1{\#D_i}\sum_{\vx\in D_i}\mathrm{Bert}(\vx)$. This defines a simple and cheap to compute selection function $c(\vx) \in \{1\dots k\}$, assigning $\vx$ to its closest centroid, i.e.,  $c(\vx) = \argmin_i\|\mathrm{Bert}(\vx) - \vb_i\|$ for $\vx\in D_{\mathrm{gen}}\cup D_{\mathrm{spec}}$. We use it to predict the closest generic data domain for each data instance from the specific set. The importance sampling weights are obtained as the ratio of examples falling in each bin:
\begin{equation}\label{equ:IS-weights}
    \valpha^{\mathrm{IS}}_i \triangleq \frac{\#\{\vx\in D_{\mathrm{spe}}| c(\vx) = i\}}{\# D_{\mathrm{spe}}}
\end{equation}
One of the main advantages of this method is its simplicity: the computation of the weights $\valpha^{\mathrm{IS}}$ is decoupled from model optimization and can be performed before training. 
It is expected to work well when the specialization set can be well approximated by the reweighted generic set, i.e., when $
    L_{\mathrm{spe}}(\vtheta)\simeq L_{\mathrm{gen}}(\vtheta, \valpha^{\mathrm{IS}}).
$
When this is not the case, it might not lead to a good specific loss.
Another potential issue with this method arises when it assigns a large weight to a generic domain $D_i$ with little available data.
In this case, training a model on $\mathrm{mix}(\valpha^{\mathrm{IS}})$ will overfit on that domain $D_i$, and it would have been better to reduce the weight of that domain to mitigate overfitting.
A last issue arises when the number of specific examples, $\#D_{\mathrm{spe}}$, is significantly smaller than the number of domains $k$. In this situation, the importance weights become sparse, as they can have at most $\#D_{\mathrm{spe}}$ non-zero coefficients. This sparsity could be problematic, as some domains with zero weights might still be close to $D_{\mathrm{spe}}$.
We illustrate these shortcomings in our experiments and explain how gradient alignment methods --- which we introduce next --- overcome them.

\subsection{DGA: Dynamic Gradient Alignment} 
\textbf{Algorithm.} We introduce the DGA: Dynamic Gradient Alignment method for data reweighting to approximately solve the bilevel problem in \autoref{equ:bilevel-form}.
This algorithm builds upon DoGE~\citep{fan2024doge} and we give a precise account of their differences later.
DGA keeps track of the model's parameters $\vtheta^{t}$ and dynamic weights $\valpha^{t}$.
Once every $T_r$ steps, we compute the gradient alignments $\va^t$, by doing 
\begin{equation}
\label{eq:alignments}
 \va^t_i =\langle \nabla \ell(\vtheta^t, \vx_i), \nabla \ell(\vtheta^t, \vz)\rangle \text{ where } \vx_i\sim D_i 
    \text{ and }\vz\sim D_{\mathrm{spe}}.
\end{equation}
and update the weights by mirror descent on the simplex~\citep{beck2003mirror} with step $\eta>0$:
\begin{equation}
\label{eq:update_mirror}
        \valpha^{t+1} = \frac{\hat{\valpha}}{\sum_{i=1}^k \hat{\valpha}_i}\text{ where }\hat{\valpha} = \valpha^t\odot \exp(\eta \va^t)
\end{equation}
We optionally store an EMA version of the weights $\valpha^t$ parameterized by $\beta \in [0, 1]$ to stabilize the training dynamics of the model's parameters, and define $\valpha_{\mathrm{EMA}}^{t+1} = (1-\beta) \valpha^t_{\mathrm{EMA}}+\beta \valpha^{t+1}$.
Finally, at each step, we update the model's parameters $\vtheta^t$ by doing an optimization step on $L_{\mathrm{gen}}(\vtheta, \valpha^t_{\mathrm{EMA}})$. The full algorithm pseudo-code is given in \autoref{alg:dga}.

\textbf{Rationale.} This algorithm can be seen as a heuristic to solve the bilevel problem in \autoref{equ:bilevel-form}.
Indeed, each update on $\vtheta$ optimizes the inner loss. 
The update rule on $\valpha$ can be seen as a mirror-descent step on $L_{\mathrm{spe}}(\vtheta^*(\valpha))$ with several approximations. The first approximation consists of approximating the solution of the inner problem with one gradient descent step with step-size $\rho$: $\vtheta^*(\valpha)\simeq \vtheta^t - \rho \sum_{i=1}^k\valpha_i\nabla L_i(\vtheta^t)$.
We then approximate the specific loss at $\vtheta^*$ by the post-update specific loss function: $L_{\mathrm{spe}}(\vtheta^*(\valpha))\simeq f(\valpha, \rho) \triangleq L_{\mathrm{spe}}(\vtheta^t - \rho \sum_{i=1}^k\valpha_i\nabla L_i(\vtheta^t)) $, that is, the drop on the specific loss after an update.
When the step size $\rho$ is small, a Taylor expansion gives
\begin{equation}
    f(\valpha, \rho) = L_{\mathrm{spe}}(\vtheta^t) - \rho \sum_{i=1}^k \valpha_i\langle \nabla L_i(\vtheta^t), \nabla L_{\mathrm{spe}}(\vtheta^t)\rangle +o(\rho)
\end{equation}
Similarly, we get that the gradient of $f$ is the gradient alignment:
\begin{equation}
    \frac{\partial f}{\partial \valpha_i}(\valpha, \rho) = -\rho \langle \nabla L_i(\vtheta^t), \nabla L_{\mathrm{spe}}(\vtheta^t)\rangle  + o(\rho)
\end{equation}
We want to use this gradient of $f$ to implement a mirror-descent method. Unfortunately, the gradients involved in the alignment are full-batch, so we approximate them with stochastic gradients obtained from mini-batches, yielding the alignments $\va^t$ from \autoref{eq:alignments}.
Overall, we get the approximation $\nabla_\valpha L_{\mathrm{spe}}(\vtheta^*(\valpha)) \simeq - \rho \va^t$; and the update rule in  \autoref{eq:update_mirror} is a mirror descent step with this approximated gradient and step $\eta /\rho$.

We have explained the link between our algorithm and the bilevel problem in \autoref{equ:bilevel-form}.
Proofs showing convergence of our method require assumptions violated in practice, e.g. 
most theoretical work assumes that the function $\vtheta\to L_{\mathrm{gen}}(\vtheta, \valpha)$ is 
convex~\citep{ghadimi2018approximation,arbel2021amortized,dagreou2022framework}. Nevertheless, successful
applications of related bilevel algorithms to non-convex neural networks have been reported recently~\citep{fan2024doge,grangier2023adaptive}.
\begin{algorithm}[ht!]
   \caption{Dynamic Gradient Alignment method}
   \label{alg:dga}
\begin{algorithmic}[1]
   \State {\bfseries Input:} Generic domains $D_1, \dots, D_k$, specific set $D_{\mathrm{spe}}$, inner optimizer state $\vomega^0$, optimizer function $\texttt{Optimizer}$ such as Adam or SGD, initial weights $\valpha^0$, outer learning rate $\eta$, EMA parameter $\beta$, weight update frequency $T_r$
   \vspace{0.2em}
   \State \textbf{Initialize EMA weights}: $\valpha_{\mathrm{EMA}}^0=\valpha^0$
   \vspace{0.2em}
   \For{$t = 0 \dots T$}
    \vspace{0.2em}
        \State Sample a batch from EMA generic mixture: $\vx \sim \mathrm{mix}(\valpha_{\mathrm{EMA}}^t)$
        \vspace{0.2em}
        \State Update the parameters $\vtheta^{t+1}, \vomega^{t+1} \leftarrow \texttt{Optimizer}(\vtheta^t, \vomega^t, \nabla_{\vtheta} \ell(\vtheta^t, \vx))$
        \vspace{0.2em}
        \If{$t \% T_r = 0$}
            \vspace{0.2em}
            \State Sample a batch from each domain: $\vx_i\sim D_i$ for $i=1\dots k$ and $\vy \sim D_{\mathrm{spe}}$
            \vspace{0.2em}
            \State Compute gradient alignements $\va^t_i\leftarrow \langle \nabla \ell(\vtheta^{t+1}, \vx_i), \nabla \ell'(\vtheta^{t+1}, \vy)\rangle$
            \vspace{0.2em}
            \State Update instantaneous weights: $\valpha^{t+1} \leftarrow\frac{\hat{\valpha}}{\sum_{i=1}^k \hat{\valpha}_i} $ with $\hat{\valpha} = \valpha^t\odot \exp(-\eta \va^t)$
            \vspace{0.2em}
            \State Update EMA weights: $\valpha_{\mathrm{EMA}}^{t+1} \leftarrow (1-\beta) \valpha_{\mathrm{EMA}}^{t} + \beta \valpha^{t+1}$
            \vspace{0.2em}
        \Else{}
            \vspace{0.2em}
            \State Do nothing: $\valpha_{\mathrm{EMA}}^{t+1} \leftarrow \valpha_{\mathrm{EMA}}^{t}$, and $\valpha^{t+1}\leftarrow \valpha^{t}$
            \vspace{0.2em}
        \EndIf
   \EndFor
   \State \textbf{Return} Optimized parameters $\vtheta^{(T)}$ and weights trajectory $\valpha^t, t=0\dots T$
\end{algorithmic}
\end{algorithm}

\textbf{Computational cost and memory overhead.} 
The computation cost of DGA is compared to the cost of a regular pre-training run.
For a base run iteration, the main cost is $t_g$, the cost of computing a gradient with a mini-batch $B$.
For DGA, we need to add the cost of updating the domain weights $\valpha$, which only happens every $T_{r}$ iterations. This update requires computing the $k+1$ gradients (one per domain, one for $L_{\mathrm{spe}}$). 
Hence the average cost of one iteration of DGA is $(1 + (k+1)T_{r}^{-1})t_g$. 
Therefore, DGA's compute overhead is small when $T_r$ is large compared to the number of domains $k$. 

During training, the memory is essentially used by the optimizer state, the model gradients and its activations. 
For simplicity, we assume the same precision for storing all vectors.
The optimizer state (the model parameters and the two EMA terms for Adam) and the gradients have a storage cost of $4m_g$, where $m_g$ denotes the cost of storing the model parameters. 
The cost of storing the activations during backpropagation is $m_b$. Regular pretraining with Adam therefore costs $4m_g + m_b$. DGA computes the required gradients sequentially and does not require more memory to store activations. It simultaneously stores two gradients instead of one (one domain gradient and one specific gradient): DGA, therefore, costs $5m_g + m_b$. This means that DGA memory overhead ranges from 0 (when $m_b \gg m_g$) to $25\%$ (when $m_g \gg m_b$).
DGA scales in terms of computational cost and memory.

\textbf{Comparison with \textsc{DoGE}.} 
While our method is heavily inspired by DoGE~\citep{fan2024doge}, there are several key differences. 
First, DGA samples from the mixture: the weights $\vtheta^t$ are updated using samples drawn from the mixture $\mathrm{mix}(\valpha^t)$, with the gradient $\nabla \ell(\vtheta^t, \vx)$ where $\vx\sim \mathrm{mix}(\valpha^t)$; this is the same gradient that one would use during pre-training with weight $\valpha^t$. 
In contrast, DoGE's weights are updated using a reweighted gradient $\sum_{i=1}^k\valpha^t_i\nabla \ell(\vtheta^t, \vx_i)$, where each $\vx_i$ are drawn from the domain $D_i$. 
For a fixed number of samples available at each draw, DGA's gradient estimate has a lower variance~\citep{seiffert2008resampling}. 
As explained above, DGA has a small overhead compared to regular pre-training, while DoGE updates the weights at each iteration.
These two key differences mean that DGA is much closer to regular pre-training than DoGE. For instance, DGA never requires retraining a model from scratch using the mixture weights estimated from a previous run, while this is the costly strategy used for DoGE.
Finally, the EMA strategy described above is novel.
\section{Experiments}
Our experiments focus on two challenging cases. 
First, given \textit{limited token resources} within each training domain, the model would risk overfitting with weights concentrated on a few domains. Second, given \textit{large number of training domains}, applying \textsc{DGA} on domain reweighting could introduce intractable computation overheads linearly increasing according to the domain granularity. 

\textbf{Generic Datasets and Domains.} For all the experiments, we use \texttt{Redpajama-v2} \citep{together2023redpajama} as the generic training set $D_{\mathrm{gen}}$. This is one of the largest public corpus for LLM pretraining. \texttt{Redpajama-v2} contains 30 trillion filtered and deduplicated tokens from web-crawled dumps. Since this corpus does not come pre-segmented into domains, we obtain obtain individual generic domains from $D_{\mathrm{gen}}$ with clustering. Specifically, we use the embedding-and-clustering pipeline from \cite{grangier2024specialized}.
We first embed all the training sequences $\vx \in D_{\mathrm{gen}}$ with SentenceBert (all-MiniLM-L6-v2), yielding a 384 dimensional embedding $\mathrm{Bert}(\vx)$. We then apply $k$-means clustering on the sentence embeddings into $k=64$ clusters yielding $k$ domains $D_1,\dots, D_{k}$.

To get fine-grained generic domains, we apply hierarchical clustering on the top of the first level of $k_1=64$ clusters. Specifically, each domain is further clustered once again into $64$ smaller clusters. We apply this strategy twice to get domains with granularity $k_2=64^2=4096$ and $k_3=64^3=262k$. 

\textbf{Model Architecture.} 
We train small ($125$M), medium ($350$M) and large ($750$M) models with decoder-only transformers~\citep{vaswani2017tranformers}. We adopt most of the training settings and architectures from~\citep{brown2020lm_few_shots}. Their details are provided in \autoref{app:sec:hyperparams}. 
For optimization, we use the AdamW optimizer~\citep{loshchilov2017decoupled}.
\subsection{Domain Reweighting with Limited Resources}\label{sec:token_limit}
\begin{figure*}[!hbpt]
  \centering
  \includegraphics[width=.7\linewidth]{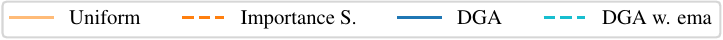} \vspace{-7pt}   \\
  \begin{subfigure}[t]{0.32\linewidth}\vskip 0pt
  \includegraphics[width=\textwidth,clip]{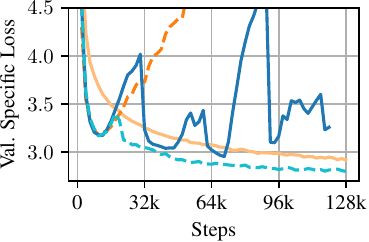}
  \caption{30M tokens per domain}
  
  \end{subfigure}\hfill
  \begin{subfigure}[t]{0.32\linewidth}\vskip 0pt
  \includegraphics[width=\textwidth,clip]{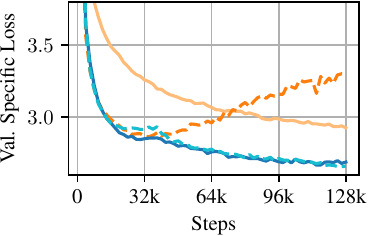}
  \caption{0.1B tokens per domain}
  
  \end{subfigure}\hfill
  \begin{subfigure}[t]{0.32\linewidth}\vskip 0pt
  \includegraphics[width=\textwidth,clip]{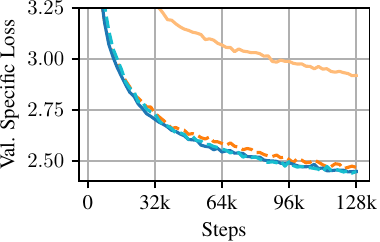}
  \caption{No token limit}
  
  \end{subfigure}
\caption{Comparing data reweighting methods with \textit{free\_law} as a specific set in a low generic data regime. When there are not enough tokens, importance sampling quickly overfits, while DGA manages to explore the training distributions to avoid overfitting. We see the importance of the EMA to stabilize DGA in the low data regime. When there is no token limit, adding an EMA ($\beta=0.1$) does not negatively affect the performance.}
 \label{fig:lowdata}
\end{figure*}
Previous works on domain reweighting implicitly assume infinite token resources from all training domains \citep{xie2023doremi,fan2024doge,liu2024regmix} while it is not always applicable in real-world cases. 
The scenario with limited training resources is challenging for online domain reweighting. 
Indeed, if the weights are concentrated on a few domains, e.g. on a single domain $D_i$, a large model 
will quickly overfit when the number of tokens in $D_i$ is small.

We expect DGA to mitigate overfitting by dynamically adjusting the domain weights. 
Specifically, once a model starts overfitting on $D_i$, the magnitude of the gradients $\nabla \ell(\vtheta, D_i)$
decreases as its training loss $\ell(\vtheta, D_i)$ is low, i.e. the domain knowledge from $D_i$ is well-learned. 
Consequently, the corresponding gradient alignment score $\va_i = \langle \nabla \ell(\vtheta, D_i), \nabla \ell(\vtheta, D_{\mathrm{spe}})\rangle$ decreases as well and DGA explores other domains with higher alignment scores. In other words, DGA down-weights domains once they are well-learned, thereby achieving a balance between \textbf{\textit{exploration}} -- by learning from diverse data domains -- and \textbf{\textit{exploitation}}, by intensively training on the most relevant domains.

However, with limited data per domain, we remark that
DGA without EMA demonstrates drastic changes at each domain weight update, focusing heavily on one domain at a time.
Quickly changing domain weights is problematic since we want to use the same domain 
weights for $T_{r}$ steps in the future. This motivates the introduction of
the EMA update in \autoref{alg:dga}, which regularizes the model and domain weights with the previous state when it starts to overfit. 

\textbf{Experiment Setup.}
We consider the generic domain split into $k=64$ domains. We construct three scales of generic sets, 
either taking the full dataset or randomly sub-sampling $30$M, $0.1$B tokens per domain.
For the targeted specific set $D_{\mathrm{spe}}$, we use $5$ subsets from \textit{the Pile} \citep{gao2020pile} covering common specialized data types for LM applications: Math (\textit{dm\_mathematics}), Code (\textit{github}, \textit{stackexchange}), Medical (\textit{pubmed\_central}), Legal (\textit{free\_law}) and Scientific articles (\textit{arxiv}). 

We implement the importance sampling baseline described in \autoref{sec:importance_sampling}. We also compare to the \textit{uniform baseline} with the domain weights $\valpha_{\mathrm{uniform}}$ as the natural proportion of each data domain in the generic \texttt{Redpajama-v2} dataset. For importance sampling and uniform baselines, the domain weights are fixed throughout the entire training run. For both vanilla DGA and DGA with an EMA ($\beta=0.1$), we update domain weights $\valpha$ every $T_{r} = 100$ steps. We use $125$M models.

\textbf{Results.} We report the validation loss on the specialized set under various token constraints in \autoref{fig:lowdata} for \textit{free\_law} and the results on other domains in \autoref{app:sec:lowdata}. 
With $30$M tokens per domain, DGA with EMA effectively stabilizes the training, while vanilla DGA exhibits several loss spikes, suggesting a lack of robustness. Under a $0.1$B token constraint, both DGA and DGA with EMA are able to dynamically adjust domain weights to mitigate overfitting. In contrast, fixed domain weights from importance sampling consistently lead to overfitting in token-limited scenarios, demonstrating the limitations of static weighting strategies in comparison to dynamic approaches like DGA. It is worth noting that adding the EMA has no negative effect on the learning efficacy when there is no token limit, which can be used as a robust regularization in the online domain reweighting context.

\textbf{Domain Weights Evolution.}
In the experiments with a limited generic token budget (\autoref{sec:token_limit}), DGA without EMA often assigns excessive weight to one generic domain, leading to overfitting due to the restricted number of training tokens. 
This iterative over-weighting pattern on generic domain weights aligns with the observed loss spikes on the specific set (\autoref{fig:limit_rw}). In contrast, the EMA helps to regularize the weight dynamics, effectively preventing the model from overfitting by maintaining more balanced domain weights throughout the training process.

\subsection{Distribution Reweighting: Scaling-up Data Mixing on Extremely Fine-grained Data Domains}
\label{sec:expe_fine_domains}
The computational overhead from \textsc{DGA} scales linearly with the number of domain $k$. This is intractable for 
datasets segmented in many fine-grained domains and, consequently, prior domain reweighting methods \citep{xie2023doremi, fan2024doge, liu2024regmix, kang2024autoscale} have not been applied in that setting.
The fine-grained setting motivates \textit{distribution reweighting} as an alternative to direct \textit{domain reweighting}.

\textit{Distribution reweighting} leverages the strength from both embedding-based (importance sampling) and gradient-based (DGA) strategies. We consider a generic training set partitioned into $k$ domains with a large $k$ (e.g. $4096, 262k$). We also have a set of $N$ auxiliary datasets $\{S_1, \hdots, S_N\}$, called \emph{basis sets}, each from a specific domain of interest. We compute the importance sampling histograms for each basis set as $P=\{\vp_1, \hdots, \vp_N\}$, $\vp_i \in \Delta^{k}, P \in R^{k\times N}$.
We then use \textsc{DGA} to search over a reparameterized space leveraging this basis. We define the 
\textit{domain weights} $\valpha_{\mathrm{domain}}\in \Delta^{k}$ as a convex combination of $N$ $k-$dimensional distributions derived from importance sampling,
\begin{equation}\label{equ:dist-reweight}
    \valpha_{\mathrm{domain}} \approx P\valpha_{\mathrm{dist}}  = \alpha_{\mathrm{dist},1}\cdot \vp_1 + \alpha_{\mathrm{dist},2}\cdot \vp_2 + \hdots + \alpha_{\mathrm{dist},N}\cdot \vp_N
\end{equation}
where the low-dimensional weights $\valpha_{\mathrm{dist}}\in\Delta^N$ are learned by DGA.
This allows the use of fine-grained domain features while eliminating intensive gradient computation on each generic domain. 
Importantly, this is equivalent to applying DGA with the $N$ generic domains $\tilde{D}_1,\dots, \tilde{D}_N$ where $\tilde{D}_i = \mathrm{mix}(\vp_i)$.
Hence, it does not require any modification to the base DGA algorithm; it suffices to be able to sample according to each $\mathrm{mix}(\vp_i)$. We provide the pseudo-code for the distribution reweighting with DGA in \autoref{app:sec:distribution_reweighting_algorithm}.

\textbf{Experiment Setup.}
We demonstrate the efficacy of distribution reweighting on the MMLU benchmark \citep{hendrycks2021measuringmassivemultitasklanguage}. MMLU consists of 57 tasks from various knowledge fields, which serves as a testbed of multi-domain language modeling; 
by measuring the downstream accuracy, we can assess whether the improvements obtained in language modeling transfer to reasoning abilities. 

We construct two specific datasets with different amounts of accessible samples: (1) \texttt{MMLU$\_$a}: we take half of the examples from each task used as $D_{\mathrm{spe}}$. We denote the other half of datapoints as \texttt{MMLU$\_$b}, which is used for evaluation; (2) \texttt{MMLU$\_$dev}: we randomly select 5 samples from each task, simulating the few-shot learning scenario.
\texttt{MMLU$\_$a} has $7.1k$ samples while \texttt{MMLU$\_$dev} only has $285$ samples, which yields sparse importance sampling histograms. 
For evaluation, we assess the language modeling performance by computing perplexity on \texttt{MMLU$\_$b}. We also measure the accuracy for multiple choice question answering on \texttt{MMLU$\_$b} with llm-eval \citep{eval-harness}.

We use generic domain splits with $k=64, 4096, 262k$ domains.
We rely on $22$ auxiliary subdomains from \textit{The Pile} \citep{gao2020pile} as our basis sets. For each auxiliary set, 
we take $15M$ tokens and compute their importance-sampling histograms as $\vp_1,\dots, \vp_N\in\Delta^k$.
To search for the optimal balance between diversity and specificity, we extend the basis sets with the importance sampling histogram from the specific set itself (i.e. \texttt{MMLU$\_$a} or \texttt{MMLU$\_$dev}), yielding $N=23$ distributions.
For this experiment, we use $750$M models.

\textbf{DGA with distribution reweighting greatly improves language modeling.} We report the loss on \texttt{MMLU$\_$b} and average validation loss across 22 domains in the Pile in \autoref{fig:dist_reweighting}. 
Both importance sampling and DGA significantly outperform the uniform baseline. 
Compared to importance sampling, we observe that DGA with distribution reweighting leads to a better Pareto-front, indicating a better balance between specialized (MMLU) and general knowledge (The Pile).
With a large domain granularity ($k=262k$ domains), training with importance sampling greatly suffers from the sparse histograms, leading to significant performance degradation. In contrast, DGA can consistently provide satisfying domain weight estimation.

\begin{figure*}[t]
  \centering
  \includegraphics[width=.7\linewidth]{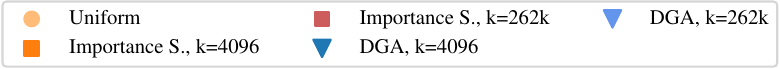} \vspace{-7pt}   \\
  \begin{subfigure}[t]{0.35\linewidth}\vskip 0pt
  \includegraphics[width=\textwidth,clip]{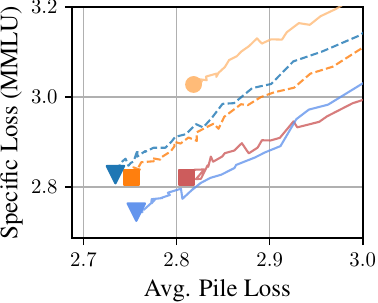}
  \caption{\texttt{MMLU$\_$a} (half the examples)}
  
  \end{subfigure}
  \begin{subfigure}[t]{0.35\linewidth}\vskip 0pt
  \includegraphics[width=\textwidth,clip]{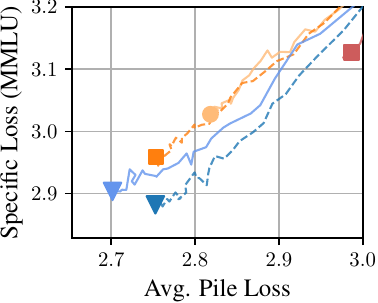}
  \caption{\texttt{MMLU$\_$dev} (5 examples per task)}
  
  \end{subfigure}\hfill
  \caption{Distribution reweighting experiment.}
 \label{fig:dist_reweighting}
\end{figure*}
\paragraph{Language modeling $\neq$ reasoning accuracy.} According to \autoref{tab:mmlu_acc}, both importance sampling and DGA reweighting greatly outperform the uniform baseline, while DGA does not show significant improvement above importance sampling despite the great improvement in language modeling. It indicates that better language modeling performance may not be necessarily transferable to better reasoning abilities. We report the full results with different model scales in \autoref{app:sec:distribution_reweighting}.
\begin{table}[!ht]
  \caption{MMLU accuracies with domain reweighting methods. Both importance sampling and DGA reweighting greatly improve the accuracy above uniform baseline, while DGA does not show significant improvement above importance sampling. \label{tab:mmlu_acc}}
   \centering
   \vspace{0.4em}
  \begin{adjustbox}{max width=0.85\textwidth}
  \begin{tabular}{ll|rr}
    \toprule
\multicolumn{2}{l}{Method} & \texttt{MMLU\_a} & \texttt{MMLU\_dev} \\
\midrule
Uniform&  &  26.1  $\%$ & 26.1  $\%$ \\
\midrule
Importance S. & $k=4096$    &  27.7 $\%$ & 27.7  $\%$ \\
              & $k=262k$    &  28.4 $\%$ & 27.0  $\%$ \\
\midrule
DGA dist. reweighting & $k=4096$ &  26.8 $\%$ & 27.4 $\%$\\
                      & $k=262k$ &  28.0 $\%$ &  27.0 $\%$\\
    \bottomrule
  \end{tabular}
  \end{adjustbox}
\end{table}
\begin{figure*}[ht]
  \centering
  \begin{subfigure}[t]{0.32\linewidth}\vskip 0pt
  \includegraphics[width=\textwidth,clip]{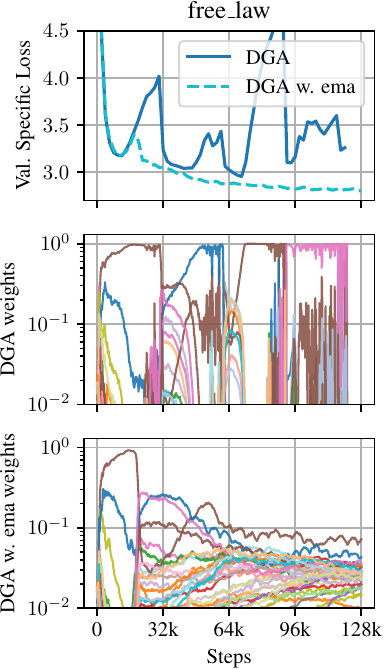}
  \caption{\textsc{DGA} w. target at \texttt{free\_law} ($30$M tokens per domain)}
  \label{fig:limit_rw}
  \end{subfigure}\hfill
  \begin{subfigure}[t]{0.32\linewidth}\vskip 0pt
  \includegraphics[width=\textwidth,clip]{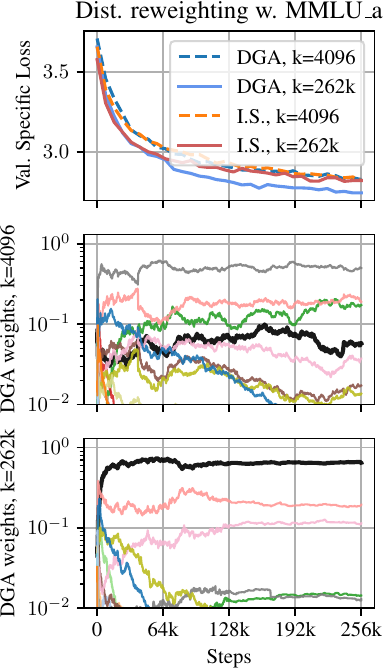}
  \caption{\textsc{DGA} w. target at \texttt{MMLU$\_$a} }
  \label{fig:dist_a}
  \end{subfigure}\hfill
  \begin{subfigure}[t]{0.32\linewidth}\vskip 0pt
  \includegraphics[width=\textwidth,clip]{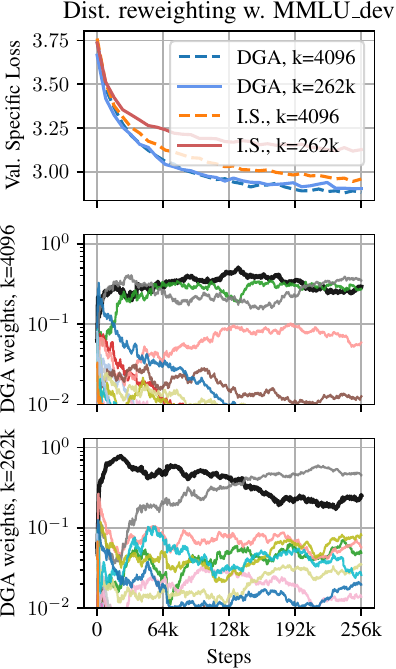}
  \caption{\textsc{DGA} w. target at \texttt{MMLU$\_$dev}}
  \label{fig:dist_dev}
  \end{subfigure}\hfill
\caption{The top row presents the specific loss over time, with the two bottom rows illustrating the evolution of domain (dist.) weights from DGA correspondingly, with each line representing a distinct domain. 
\textbf{Left}: Weights from the limited generic token experiment (\autoref{sec:token_limit}). \textbf{Middle} and \textbf{Right}: Weights from the distribution reweighting experiment (\autoref{sec:expe_fine_domains}). The thick black line highlights the dynamic weights assigned by DGA on the MMLU importance sampling distribution, which serves as a fixed training distribution for the importance sampling runs.}
 \label{fig:histograms}
\end{figure*}

\textbf{Weights Evolution on Distributions.}
We present the evolution of domain weights for each basis distribution from DGA in \autoref{fig:histograms}. Comparing different levels of granularity, with $k$=$262k$, the importance of the MMLU distribution is more emphasized than with $k$=$4096$, with the help of fine-grained domain features. 
Additionally, with sufficient samples from the specific domain (\texttt{MMLU\_a}, \autoref{fig:dist_a}), the MMLU distribution is consistently up-weighted across $262k$ generic domains. 
In contrast, on \texttt{MMLU\_dev}, while the distribution on MMLU is initially up-weighted, it declines gradually in the late stage of training.
Owing to the number of accessible samples from the specific set, the importance sampling distribution on \texttt{MMLU\_dev} across $262k$ generic domains is very sparse.
During the training, the learnability of the few activated generic domains diminishes, making other distributions more beneficial to the model.

In addition to the importance sampling distribution from the specific sets (\texttt{MMLU\_a} and \texttt{MMLU\_dev}), \textsc{DGA} effectively identifies other relevant distributions from The Pile that contribute to the learning on MMLU. These influential distributions, which include \texttt{phil\_papers}, \texttt{free\_law}, and \texttt{dm\_mathematics}, are all considered to contain high-quality, academic-related contents. We present detailed curves with domain labels in \autoref{app:sec:weights_dr}.
This ability to adaptively select beneficial distributions enhances the model’s generalization and helps mitigate overfitting by leveraging a broader yet pertinent set of data sources during pretraining.

\textbf{Impact of generic domain granularity.} 
In \autoref{fig:granularity}, we present the validation loss on the specific domain according to the number of clusters within the generic dataset. From $k=64$ to $4096$, both \textsc{DGA} and importance sampling demonstrate significant improvement in language modeling in terms of validation loss (i.e., log of perplexity). However, when the number of clusters exceeds the scale of the accessible specific set, the importance sampling method overfits the limited number of activated generic domains, failing to capture broader domain knowledge. In contrast, \textsc{DGA} effectively leverages extremely fine-grained domain information across $262k$ generic domains with only $7k$ samples from \texttt{MMLU\_a}. 
In the few-shot context (\texttt{MMLU\_dev}), \textsc{DGA} mitigates a large performance degradation by utilizing diverse domain knowledge from other relevant distributions.
\begin{figure*}[!ht]
  \centering
  \begin{subfigure}[t]{0.4\linewidth}\vskip 0pt
  \includegraphics[width=\textwidth,clip]{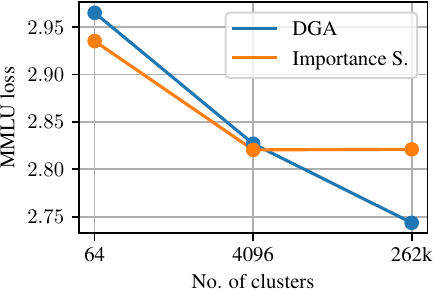}
  \caption{$D_{\mathrm{spe}}=\texttt{MMLU\_a}$}
  
  \end{subfigure}
  \begin{subfigure}[t]{0.4\linewidth}\vskip 0pt
  \includegraphics[width=\textwidth,clip]{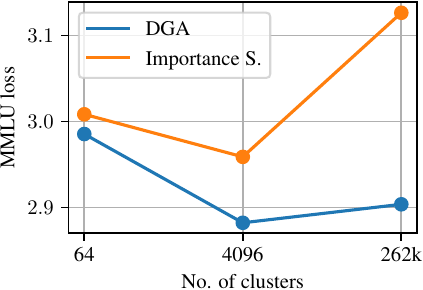}
  \caption{$D_{\mathrm{spe}}=\texttt{MMLU\_dev}$}
  
  \end{subfigure}
\caption{Impact of the generic set granularity for the distribution reweighting experiment (\autoref{sec:expe_fine_domains}. We report the specific loss obtained after training for different granularities of the base clustering.}
 \label{fig:granularity}
\end{figure*}
\section{Related Work}
\paragraph{Task-adaptive Data Selection for Domain-Specific LLMs.}
Many works have shown that one can effectively improve the LLM's performance on a specific downstream task with data selection according to the relevance of generic data for the targeted data domain. 
\citet{gururangan2020dontstoppretrainingadapt} show that continued pretraining on data with high vocabulary overlap can boost its performance on the specific end-task. 
On machine translation task, \citet{aharoni-goldberg-2020-unsupervised} identify task-relevant pretraining datasets from a generic corpus using nearest neighbor of a small specialist dataset based on SentenceBert sentence representation.
\citet{wang2020optimizing,grangier2023adaptive} train a small proxy model to give an importance weight per sample. 
\citet{xie2023dataselectionlanguagemodels} proposed DSIR as a lexical-based importance sampling method using n-gram features. 

Other than feature-based importance sampling \citep{grangier2024specialized}, influence function-based method select data points which leads to the greatest loss drop on the target from the optimization perspective \citep{koh2020understanding, kwon2024datainf, agarwal2017secondorder}. However, these methods often introduce intensive computational overheads from the second-order gradient computations, which is not applicable on large generic pretraining corpus.
\paragraph{Data Resampling through Domain Reweighting.}
Given the large scale of the generic pretraining corpus, sample-level selection strategies are hard to implement for LLM pretraining. Alternatively, domain reweighting methods \citep{xie2023doremi, fan2024doge, liu2024regmix, kang2024autoscale} apply group-level selection by adjusting data sampling weights across different domains to reflect their importance in pretraining.  
Based on the weak-to-strong generalization strategy \citep{burns2023weaktostronggeneralizationelicitingstrong}, existing domain re-weighting methods typically estimate the optimal domain weights for a larger base model based on the preferences of a small-scale proxy model. \citet{xie2023doremi} apply group distributed robust optimization to optimize the worst-case loss gap between two small-scale proxies. \citet{fan2024doge} use gradient alignment to dynamically adjust domain weights during proxy model training. Specifically, it identifies the most beneficial domains by aligning the gradients of the training data with the target task. However, it trains the proxy model on reweighted domain gradients to simulate the resampling scenario, which introduces more variance in the domain weights estimation. 

\section{Conclusion}
To tackle two key challenges of online domain reweighting, we introduce Dynamic Gradient Alignment (DGA) as a stable and scalable data mixing method for LLM pretraining. 
Given a target task, DGA is an online algorithm that adjusts the training data distribution according to the current model status. 
This adaptation relies on an estimate of the progress on the target task from gradient alignments. 
We show that under limited tokens within generic domains, DGA with EMA can notably mitigate overfitting and yields superior performance on the end-task by balancing exploitation and exploration.
We also propose a novel distribution reweighting strategy, which enables DGA to scale up to extremely fine-grained data domains without incurring intensive computations. Our experiments on MMLU show that applying distribution reweighting with DGA effectively leverages fine-grained domain knowledge to balance specialty and diversity during training. 
Our work demonstrates the scalability of gradient-alignment-based data reweighting methods, as well as their efficiency in data-constrained settings.

\subsubsection*{Acknowledgments}
We thank Angelos Katharopoulos, Skyler Seto, and Matteo Pagliardini for their help and fruitful discussions during the project.
\newpage
\bibliographystyle{abbrvnat}
\bibliography{biblio}

\newpage
\appendix
\section{Training with Limited Generic Tokens}
\label{app:sec:lowdata}
\subsection{Validation Loss on the Targeted End-task}
We present the complete results on all six target domains (\texttt{arxiv}, \texttt{free\_law}, \texttt{dm\_mathematics}, \texttt{pubmed\_central}, \texttt{github}, \texttt{stackexchange}) as follows. 
Across all six target domains, DGA with EMA ($\beta=0.1$) consistently stablize the training and yields better language modelling performance under token-limited contexts. 
\begin{figure*}[!h]
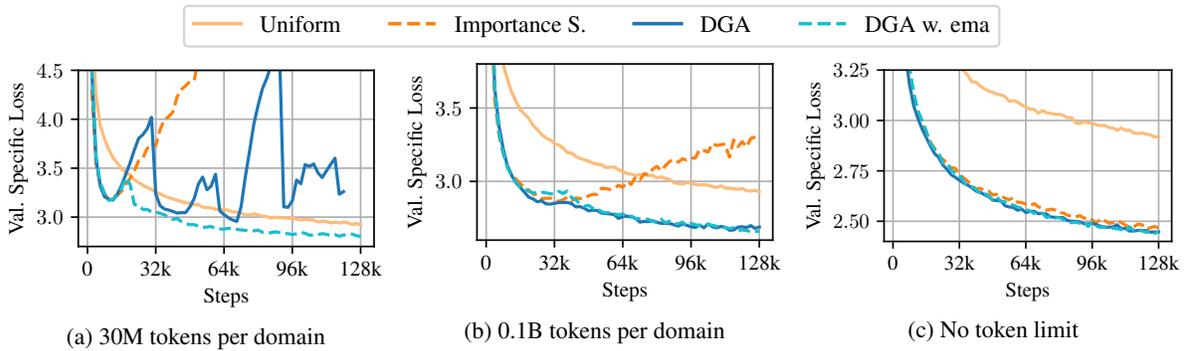

  \centering
  \includegraphics[width=.7\linewidth]{figs/lowdata/legend.pdf} \vspace{-7pt}   \\
  \begin{subfigure}[t]{0.32\linewidth}\vskip 0pt
  \includegraphics[width=\textwidth,clip]{figs/lowdata/free_law-31M.pdf}
  \caption{30M tokens per domain}
  
  \end{subfigure}\hfill
  \begin{subfigure}[t]{0.32\linewidth}\vskip 0pt
  \includegraphics[width=\textwidth,clip]{figs/lowdata/free_law-99M.pdf}
  \caption{0.1B tokens per domain}
  
  \end{subfigure}\hfill
  \begin{subfigure}[t]{0.32\linewidth}\vskip 0pt
  \includegraphics[width=\textwidth,clip]{figs/lowdata/free_law-no_limit.pdf}
  \caption{No token limit}
  
  \end{subfigure}
\caption{Results on all the domains for the low data experiment (\autoref{sec:token_limit}). The specific domain is \texttt{free\_law}.}
 
\end{figure*}

\begin{figure*}[!h]
  \centering
  \includegraphics[width=.7\linewidth]{figs/lowdata/legend.pdf} \vspace{-7pt}   \\
  \begin{subfigure}[t]{0.32\linewidth}\vskip 0pt
  \includegraphics[width=\textwidth,clip]{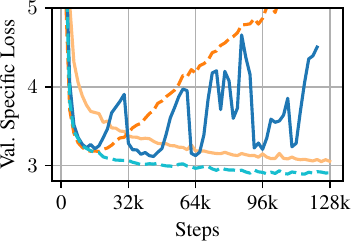}
  \caption{30M tokens per domain}
  
  \end{subfigure}\hfill
  \begin{subfigure}[t]{0.32\linewidth}\vskip 0pt
  \includegraphics[width=\textwidth,clip]{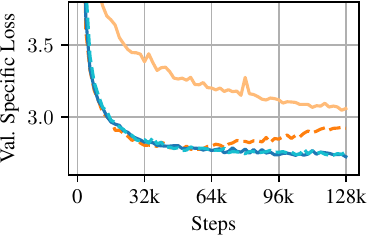}
  \caption{0.1B tokens per domain}
  
  \end{subfigure}\hfill
  \begin{subfigure}[t]{0.32\linewidth}\vskip 0pt
  \includegraphics[width=\textwidth,clip]{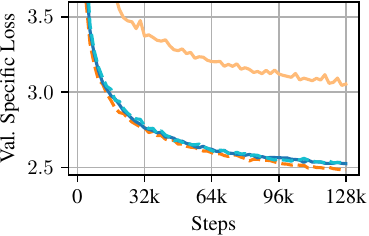}
  \caption{No token limit}
  
  \end{subfigure}
\caption{Results on all the domains for the low data experiment (\autoref{sec:token_limit}). The specific domain is \texttt{arxiv}}
 
\end{figure*}

\begin{figure*}[!h]
  \centering
  \includegraphics[width=.7\linewidth]{figs/lowdata/legend.pdf} \vspace{-7pt}   \\
  \begin{subfigure}[t]{0.32\linewidth}\vskip 0pt
  \includegraphics[width=\textwidth,clip]{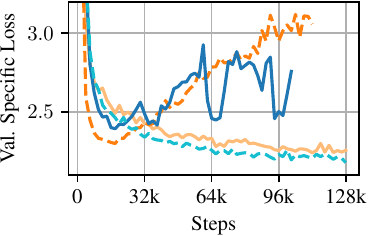}
  \caption{30M tokens per domain}
  
  \end{subfigure}\hfill
  \begin{subfigure}[t]{0.32\linewidth}\vskip 0pt
  \includegraphics[width=\textwidth,clip]{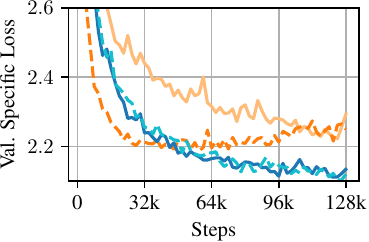}
  \caption{0.1B tokens per domain}
  
  \end{subfigure}\hfill
  \begin{subfigure}[t]{0.32\linewidth}\vskip 0pt
  \includegraphics[width=\textwidth,clip]{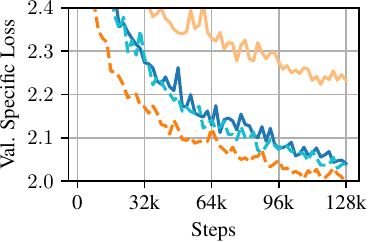}
  \caption{No token limit}
  
  \end{subfigure}
\caption{Results on all the domains for the low data experiment (\autoref{sec:token_limit}). The specific domain is \texttt{dm-mathematics}}
 
\end{figure*}

\begin{figure*}[!h]
  \centering
  \includegraphics[width=.7\linewidth]{figs/lowdata/legend.pdf} \vspace{-7pt}   \\
  \begin{subfigure}[t]{0.32\linewidth}\vskip 0pt
  \includegraphics[width=\textwidth,clip]{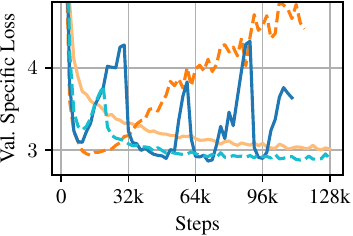}
  \caption{30M tokens per domain}
  
  \end{subfigure}\hfill
  \begin{subfigure}[t]{0.32\linewidth}\vskip 0pt
  \includegraphics[width=\textwidth,clip]{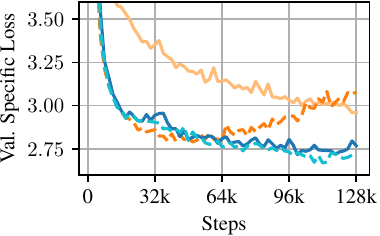}
  \caption{0.1B tokens per domain}
  
  \end{subfigure}\hfill
  \begin{subfigure}[t]{0.32\linewidth}\vskip 0pt
  \includegraphics[width=\textwidth,clip]{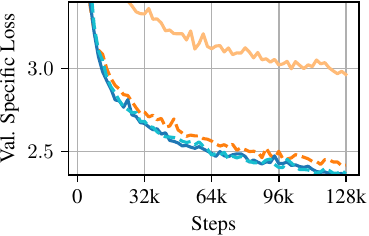}
  \caption{No token limit}
  
  \end{subfigure}
\caption{Results on all the domains for the low data experiment (\autoref{sec:token_limit}). The specific domain is \texttt{github}}
 
\end{figure*}

\begin{figure*}[!h]
  \centering
  \includegraphics[width=.7\linewidth]{figs/lowdata/legend.pdf} \vspace{-7pt}   \\
  \begin{subfigure}[t]{0.32\linewidth}\vskip 0pt
  \includegraphics[width=\textwidth,clip]{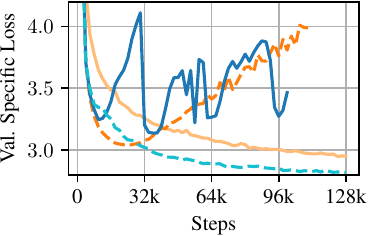}
  \caption{30M tokens per domain}
  
  \end{subfigure}\hfill
  \begin{subfigure}[t]{0.32\linewidth}\vskip 0pt
  \includegraphics[width=\textwidth,clip]{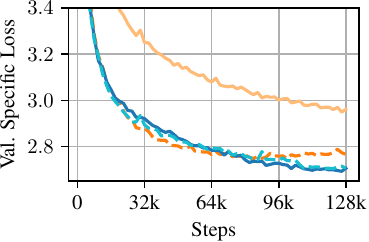}
  \caption{0.1B tokens per domain}
  
  \end{subfigure}\hfill
  \begin{subfigure}[t]{0.32\linewidth}\vskip 0pt
  \includegraphics[width=\textwidth,clip]{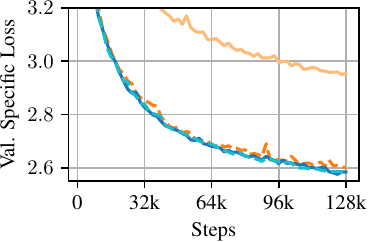}
  \caption{No token limit}
  
  \end{subfigure}
\caption{Results on all the domains for the low data experiment (\autoref{sec:token_limit}). The specific domain is \texttt{pubmed-central}}
 
\end{figure*}

\begin{figure*}[!h]
  \centering
  \includegraphics[width=.7\linewidth]{figs/lowdata/legend.pdf} \vspace{-7pt}   \\
  \begin{subfigure}[t]{0.32\linewidth}\vskip 0pt
  \includegraphics[width=\textwidth,clip]{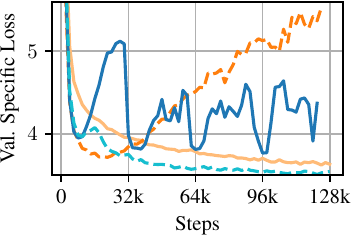}
  \caption{30M tokens per domain}
  
  \end{subfigure}\hfill
  \begin{subfigure}[t]{0.32\linewidth}\vskip 0pt
  \includegraphics[width=\textwidth,clip]{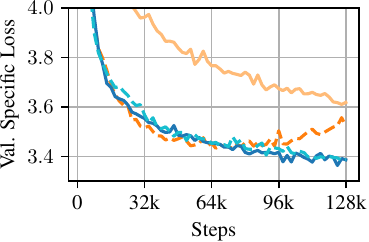}
  \caption{0.1B tokens per domain}
  
  \end{subfigure}\hfill
  \begin{subfigure}[t]{0.32\linewidth}\vskip 0pt
  \includegraphics[width=\textwidth,clip]{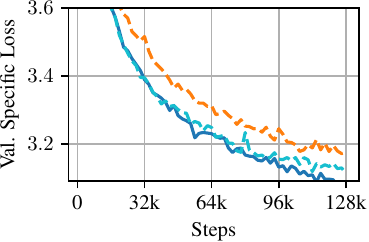}
  \caption{No token limit}
  
  \end{subfigure}
\caption{Results on all the domains for the low data experiment (\autoref{sec:token_limit}). The specific domain is \texttt{stackexchange}}
 
\end{figure*}

\subsection{Domain Weights Evolution}
\label{app:sec:domain_reweight}
We present the domain weights evolution on 64 generic domains from DGA with and w.o. EMA regularization. With both \texttt{stackexchange} and \texttt{free\_law} as the specific set, EMA effectively smoothes the spiky domain weights, which therefore stablize the training process.
\begin{figure*}[!htbp]
  \centering
  \begin{subfigure}[t]{0.44\linewidth}\vskip 0pt
\includegraphics[width=\textwidth,clip]{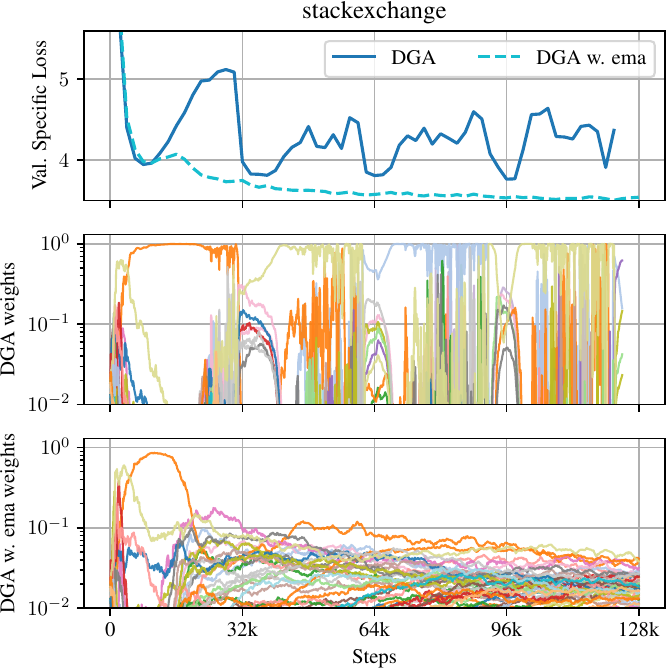}
  \caption{}
  
  \end{subfigure}\hfill
  \begin{subfigure}[t]{0.44\linewidth}\vskip 0pt
  \includegraphics[width=\textwidth,clip]{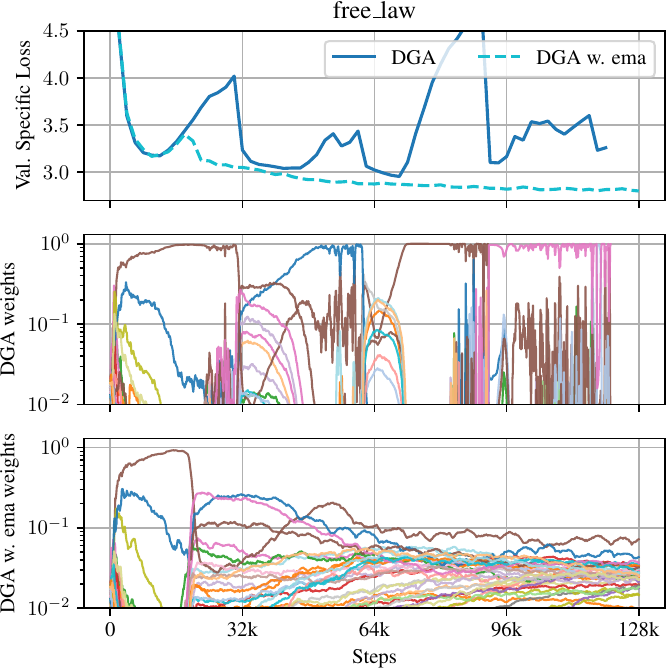}
  \caption{}
  
  \end{subfigure}\hfill
\caption{Comparing data reweighting methods with \texttt{stackexchange} (resp. \texttt{free\_law}) as the specific set, in a low generic data regime. 
When there are not enough tokens, importance sampling quickly overfits, while DGA manages to explore the training distributions to avoid overfitting. We see the importance of the EMA to stabilize DGA in the low data regime.}
 \label{fig:lowdata_freelaw}
\end{figure*}

\newpage
\section{Distribution Reweighting}
\label{app:sec:distribution_reweighting}
\subsection{Evaluation Results on MMLU}
We present the complete evaluation results on MMLU benchmark on small- ($125$M) and large- ($750$M) scale models. $k$ denotes the number of generic domains, $N$ denotes the number of reweighted importance sampling distributions from \textit{basis sets}. $N$=$22$ indicates we only reweight 22 distributions from 22 \textit{The Pile} subsets, while $N$=$23$ includes the importance sampling histgram from the specific set (MMLU). 
Since the $125$M model shows marginal difference in accuracy because of limited capacity, we only scored $750$M model on MMLU reasoning accuracies.

\begin{table}[!ht]
\centering
\caption{Results on the domain reweighting experiment, with half MMLU as train set. The best results is \textbf{Bolded} and the second best is \underline{Underlined}. \label{tab:domain_reweighting_results}}
\begin{tabular}{l| r r r}
\toprule
125M model & MMLU loss & MMLU acc. & avg. Pile loss\\
\midrule
Uniform                   & 3.56 &  - & 3.27 \\
Importance S. (k=4096)      & 3.32 &  - & 3.16 \\
Importance S. (k=262k)      & \textbf{3.22} &  - & 3.19 \\
DGA domain reweighting (k=64) & 3.31 &  - & \underline{3.10} \\
DGA dist. reweighting (N=22, k=4096) & 3.34 &  - & 3.13 \\
DGA dist. reweighting (N=22, k=262k) & 3.34 &  - & \textbf{3.05} \\
DGA dist. reweighting (N=23, k=4096) & 3.33 &  - & 3.13 \\
DGA dist. reweighting (N=23, k=262k) & \underline{3.25} &  - & \underline{3.10} \\
\midrule
750M model & MMLU loss & MMLU acc. & avg. Pile loss\\
\midrule
Uniform                   & 3.03 &  26.1 $\%$ & 2.82 \\
Importance S. k=4096      & \underline{2.82} &  27.7 $\%$ & 2.75 \\
Importance S. k=262k      & \underline{2.82} &  \textbf{28.4} $\%$ & 2.81 \\
DGA domain reweighting (k=64)                  & 2.97 &  27.1 $\%$ & 2.77 \\
DGA dist. reweighting (N=22, k=4096) & 2.86 &  27.2 $\%$ & \underline{2.73} \\
DGA dist. reweighting (N=22, k=262k) & 2.84 &  27.0 $\%$ & \textbf{2.66} \\
DGA dist. reweighting (N=23, k=4096) & 2.83 &  26.8 $\%$ & \underline{2.73} \\
DGA dist. reweighting (N=23, k=262k) & \textbf{2.74} &  \underline{28.0} $\%$ & 2.76 \\
\bottomrule
\end{tabular}
\end{table}

\begin{table}[h!]
\centering
\caption{Results on the domain reweighting experiment, with 5 examples per task of MMLU as train set. We score only the 750M models.\label{tab:domain_reweighting_results_dev}}
\begin{tabular}{l| r r r}
\toprule
125M model & MMLU loss & MMLU acc. & avg. Pile loss\\
\midrule
Uniform                   & 3.56 &  - & 3.27 \\
Importance S. (k=4096)      & 3.40 &  - & 3.16 \\
Importance S. (k=262k)      & 3.41 &  - & 3.29 \\
DGA domain reweighting (k=64)& 3.46 &  - & 3.19 \\
DGA dist. reweighting (N=22, k=4096) & 3.37 &  - & 3.12 \\
DGA dist. reweighting (N=22, k=262k) & \underline{3.36} &  - & \textbf{3.03} \\
DGA dist. reweighting (N=23, k=4096) & 3.37 &  - & 3.14 \\
DGA dist. reweighting (N=23, k=262k) & \textbf{3.35} &  - & \underline{3.08} \\
\midrule
750M model & MMLU loss & MMLU acc. & avg. Pile loss\\
\midrule
Uniform                   & 3.03 &  26.1 $\%$ & 2.82 \\
Importance S. k=4096      & 2.96 &  27.7 $\%$ & 2.75 \\
Importance S. k=262k      & 3.13 &  \underline{27.0} $\%$ & 2.99 \\
DGA k=64                  & 3.01 &  \underline{27.0} $\%$ & 2.77 \\
DGA dist. reweighting from 22 domains, k=4096 & \underline{2.89} &  26.8 $\%$ & 2.76 \\
DGA dist. reweighting from 22 domains, k=262k & 2.93 &  \underline{27.0} $\%$ & \textbf{2.68} \\
DGA dist. reweighting from 23 domains, k=4096 & \textbf{2.88} &  \textbf{27.4} $\%$ & 2.75 \\
DGA dist. reweighting from 23 domains, k=262k & 2.90 &  \underline{27.0} $\%$ & \underline{2.70} \\
\bottomrule
\end{tabular}
\end{table}

\subsection{Weights Evolution on Distributions}
\label{app:sec:weights_dr}
\begin{figure*}[!ht]
  \centering
  \begin{subfigure}[t]{0.49\linewidth}\vskip 0pt
  \includegraphics[width=\textwidth,clip]{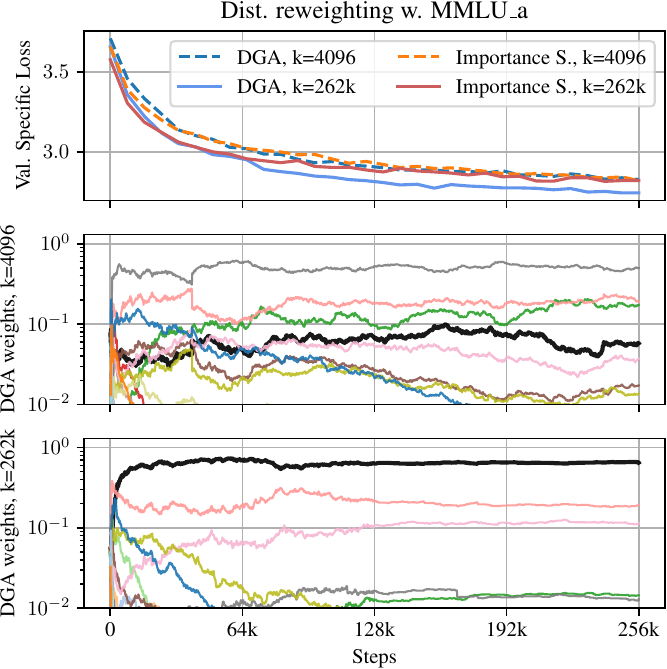}
  \caption{\textsc{DGA} w. target at \texttt{MMLU$\_$a} }
  
  \end{subfigure}\hfill
  \begin{subfigure}[t]{0.49\linewidth}\vskip 0pt
  \includegraphics[width=\textwidth,clip]{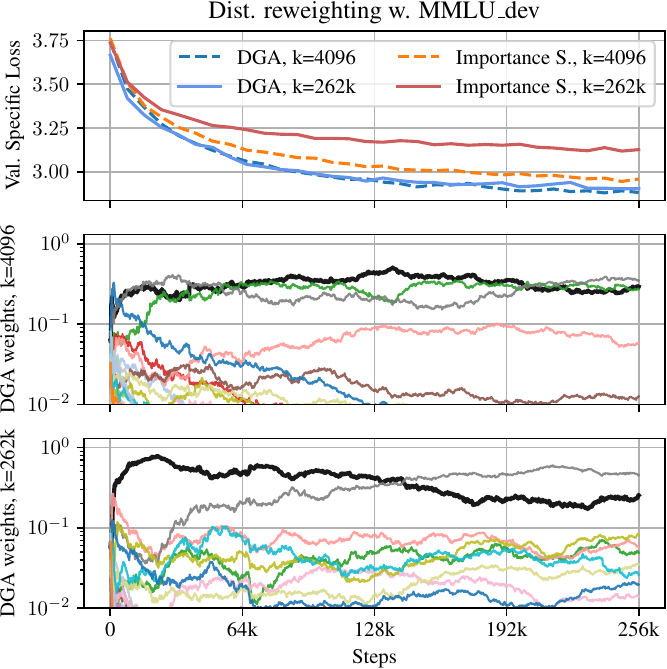}
  \caption{\textsc{DGA} w. target at \texttt{MMLU$\_$dev} }
  
  \end{subfigure}
  \begin{subfigure}[t]{0.99\linewidth}\vskip 0pt
  \includegraphics[width=\textwidth,clip]{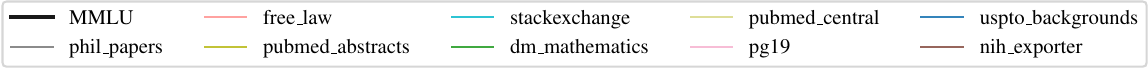}
  \caption{Top 10 upweighted distributions}
  
  \end{subfigure}\hfill
\caption{The top row presents the specific loss over time, with the two bottom rows illustrating the evolution of domain (dist.) weights from DGA correspondingly, with each line representing a distinct domain. 
}
 \label{app:fig:dr_histograms}
\end{figure*}

\newpage
\section{Hyperparameters}
\label{app:sec:hyperparams}

\autoref{tab:models} provides the model architectures and hyperparameters used in this paper.
\begin{table}[htbp!]
\caption{Architecture hyperparameters for various model scales used in the paper. All models are vanilla Transformer decoder-only models.\label{tab:models}}
\label{tab:archictectures}
\centering
\vspace{5pt}
\begin{adjustbox}{max width=0.9\textwidth}
\begin{tabular}{lcccccc}
\toprule
     & Layers & Attention heads & Embed dim & Hidden dim & Context limit & learning rate \\
     \midrule
125M & 12      & 12              & 768       & 3072  & 1024   & $1\times 10^{-4}$  \\     
350M & 24     & 16              & 1024       & 4096  & 1024  & $1\times 10^{-4}$ \\
750M & 36     & 20              & 1280       & 5120  & 1024  & $1\times 10^{-4}$  \\
\bottomrule
\end{tabular}
\end{adjustbox}
\end{table}

\section{DGA for Distribution Reweighting}
\label{app:sec:distribution_reweighting_algorithm}
\autoref{alg:dga_distribution} explains the distribution reweighting with DGA. The implementation can be easily adapted from domain reweighting \textsc{DGA} with minor modifications.
\begin{algorithm}[ht!]\label{app:alg:dist_reweight}
   \caption{Distribution Reweighting w. DGA. \small (Difference from domain reweighting are marked in \textcolor{DarkBlue}{blue})}
   \label{alg:dga_distribution}
\begin{algorithmic}[1]
   \State {\bfseries Input:} Generic domains $D_1, \dots, D_k$, I.S. distributions $\mathcal{A}_{dist}\triangleq[\vp_1,\dots, \vp_N]$, specific set $D_{\mathrm{spe}}$, inner optimizer state $\vomega^0$, optimizer function $\texttt{Optimizer}$ such as Adam or SGD, initial weights $\valpha^0$, outer learning rate $\eta$, weight update frequency $T_r$
   \vspace{0.2em}
   \State \textbf{Initialize \textcolor{DarkBlue}{distribution weights}}: $\valpha_{\mathrm{dist}}^0=\valpha^0$, i.e. init. \textcolor{DarkBlue}{domain weights}: $\valpha_{\mathrm{domain}}^0=\valpha_{\mathrm{dist}}^0 \otimes \mathcal{A}_{dist}$.
   \For{$t = 0 \dots T$}
        \vspace{0.2em}
        \State Sample batch from generic mixture: $\vx \sim \mathrm{mix}(\valpha_{\mathrm{domain}}^t)$
        \vspace{0.2em}
        \State Update the parameters $\vtheta^{t+1}, \vomega^{t+1} \leftarrow \texttt{Optimizer}(\vtheta^t, \vomega^t, \nabla_{\vtheta} \ell(\vtheta^t, \vx))$
        \vspace{0.2em}
        \If{$t \% T_r = 0$}
            \vspace{0.2em}
            \State Sample a batch from each \textcolor{DarkBlue}{\emph{distribution}}: \textcolor{DarkBlue}{$\vx_i\sim \mathrm{mix} (\vp_i)$ for $i=1\dots N$} and $\vy \sim D_{\mathrm{spe}}$
            \vspace{0.2em}
            \State Compute gradient alignements $\va^t_i\leftarrow \langle \nabla \ell(\vtheta^{t+1}, \vx_i), \nabla \ell'(\vtheta^{t+1}, \vy)\rangle$
            \vspace{0.2em}
            \State Update \textcolor{DarkBlue}{\emph{distribution weights}}: \textcolor{DarkBlue}{$\valpha_{\mathrm{dist}}^{t+1} \leftarrow\frac{\hat{\valpha}}{\sum_{i=1}^k \hat{\valpha}_i}$} with $\hat{\valpha} = \valpha_{\mathrm{dist}}^t\odot \exp(-\eta \va^t)$, 
            \vspace{0.2em}
            \State Updated \textcolor{DarkBlue}{\emph{domain weights}}: $\valpha_{\mathrm{domain}}^{t+1}=\valpha_{\mathrm{dist}}^{t+1} \otimes \mathcal{A}_{dist}$.
            \vspace{0.2em}
        \Else{}
            \vspace{0.2em}
            \State Do nothing: $\valpha_{\mathrm{dist}}^{t+1} \leftarrow \valpha_{\mathrm{dist}}^{t}$
            \vspace{0.2em}
        \EndIf
   \EndFor
   \State \textbf{Return} Optimized parameters $\vtheta^{(T)}$ and weights trajectory $\valpha^t, t=0\dots T$
\end{algorithmic}
\end{algorithm}
\end{document}